\documentclass{article}

\PassOptionsToPackage{numbers}{natbib}

     \usepackage[preprint]{neurips_2022}

\usepackage[utf8]{inputenc} 
\usepackage[T1]{fontenc}    
\usepackage{hyperref}       
\usepackage{url}            
\usepackage{booktabs}       
\usepackage{amsfonts}       
\usepackage{nicefrac}       
\usepackage{microtype}      
\usepackage{xcolor}         

\usepackage{amsmath}
\usepackage{amsthm}
\usepackage{amssymb}
\usepackage[bb=dsserif]{mathalpha}
\usepackage{dsfont}
\usepackage{siunitx}

\usepackage{caption}
\usepackage{subcaption}
\usepackage{multirow}
\usepackage{float}

\usepackage[noend]{algpseudocode}
\usepackage{algorithm}
\algrenewcommand{\Function}[2]{\State{{\bf function} \textproc{#1}(#2)}}
\algrenewcommand{\algorithmiccomment}[1]{\hspace{\algorithmicindent}\(\triangleright\) #1}

\usepackage{tikz}
\usetikzlibrary{positioning,arrows.meta}

\newtheorem{prop}{Proposition}
\newtheorem{lem}[prop]{Lemma}
\newtheorem*{prop1}{Proposition 1}
\newtheorem*{defn}{Definition}

\DeclareMathOperator{\sgn}{sgn}
\newcommand{\coloneqq}{\mathrel{\rlap{%
                     \raisebox{0.5ex}{$\cdot$}}%
                     \raisebox{-0.5ex}{$\cdot$}}%
                     =}

\title{Robustness Evaluation and Adversarial Training of an Instance Segmentation Model}

\author{%
  Jacob Bond\\
  General Motors Company\\
  \texttt{jacob.bond@gm.com} \\
  \And
  Andrew Lingg\\
  General Motors Company\\
  \texttt{andrew.lingg@gm.com} \\
}

\begin{document}

\maketitle

\begin{abstract}
To evaluate the robustness of non-classifier models, we propose probabilistic local equivalence, based on the notion of randomized smoothing, as a way to quantitatively evaluate the robustness of an arbitrary function.  In addition, to understand the effect of adversarial training on non-classifiers and to investigate the level of robustness that can be obtained without degrading performance on the training distribution, we apply Fast is Better than Free adversarial training together with the TRADES robust loss to the training of an instance segmentation network.  In this direction, we were able to achieve a symmetric best dice score of 0.85 on the TuSimple lane detection challenge, outperforming the standardly-trained network's score of 0.82.  Additionally, we were able to obtain an F-measure of 0.49 on manipulated inputs, in contrast to the standardly-trained network's score of 0.  We show that probabilisitic local equivalence is able to successfully distinguish between standardly-trained and adversarially-trained models, providing another view of the improved robustness of the adversarially-trained models.
\end{abstract}

\section{Introduction}

While deep learning methods are able to achieve excellent performance on inputs originating from the same data distribution as the training data, generalization to additional distributions, including data which has been intentionally manipulated, remains a challenge for deep learning models.  The poor generalizability of these models is a concern both in safety-critical situations, for example as deployment as part of a warehouse robot or an advanced driver-assistance system, but also in situations of great social or economic consequence, such as in financial credit assessment \cite{blattner21} or recidivism algorithms \cite{angwin16}.

On the concern of input manipulation \cite{bimltaxonomy}, adversarial training techniques have made great strides since their principled formulation by M\c{a}dry et al.\  \cite{madry18}, however, a large proportion of the work has focused on object classification and, to a lesser extent, object detection.  There has been significantly less work investigating the adversarial training and robustness of generic machine learning tasks, despite the fact that many of the security-critical applications which demand adversarial robustness lie outside of the object classification domain.

Additionally, evaluation of a model's robustness typically focuses on evaluating the adversarial robustness of the model relative to the training data distribution, despite the frequent tendency for models to see a different data distribution in production.  There is a need for additional methods for evaluating a model's robustness beyond just its performance on manipulated and unmanipulated inputs from the training distribution.

To this end, we reformulated randomized smoothing \cite{cohen19} as a robustness metric.  In order to analyze the efficacy of the robustness metric, we compared the performance of standardly-trained instance segmentation models to models trained using two different approaches to adversarial training.  Specifically, we adversarially trained a lane line instance segmentation model on the TuSimple lane detection challenge \cite{tusimple17}.  By adding a TRADES loss \cite{zhang19} regularization for the segmentation prediction to the adversarial training, we were able to improve the average symmetric best dice score of the model from 0.82 to 0.85, while also improving the average F-measure on manipulated imputs from 0 to 0.49.  We found that the proposed robustness metric was able to differentiate between each of the three models, indicating its effectiveness.

\section{Related Work}

Since the establishment of adversarial training through a principled formulation by M\c{a}dry \cite{madry18}, a number of improvements have been proposed.  Although early attempts at adversarial training using the fast gradient sign method (FGSM) \cite{goodfellow15} to generate perturbations were at best mildly successful \cite{madry18, tramer18}, Wong et al.\ \cite{wong20} was able to establish that FGSM-based adversarial training with randomized initializations achieved comparable results to adversarial training based on projected gradient descent (PGD).  The Fast is Better than Free (FBF) approach proposed by Wong et al.\ crucially reduced the training time of adversarial training to approximately double that of standard training, resulting in an overhead that is typically manageable even for production systems.

Another potentially significant drawback of early approaches to adversarial training was a reduction in accuracy on the original training distribution.  This reduction has since been mitigated to a significant extent.  Zhang et al.\ \cite{zhang19} augment the standard loss function of a classifier by adding a robust loss term meant to encourage the model to return similar results for both natural and manipulated inputs.  Carmon et al.\ \cite{carmon19} utilized a two-stage approach using a standardly-trained model to label a collection of unlabeled data which was then added to the data available for adversarial training.  This approach helped to encourage a model's outputs to be stable around natural inputs.  Gowal et al.\ \cite{gowal21} investigated a number of these improvements in combination, albeit sometimes requiring significant changes from the original model architecture or training procedure, finding that adding TRADES loss, model weight averaging, and unlabeled data all provide meaningful improvements to both the standard and adversarial performance of a robust model.

While adversarial training undeniably makes a model more robust, it remains a challenge to truly identify the impact on a model's performance beyond natural and adversarial performance on the training distribution.  One way to gauge a model's robustness to \(L^{p}\)-norm-bounded perturbations is by plotting a security curve across a range of \(L^{p}\)-norm perturbation bounds, as in Adversarial Robustness Toolbox \cite{art}.  Alternatively, AutoAttack \cite{croce20} provides a parameter-free method for evaluating a model's robustness to input manipulations, though it is largely limited to object classification models.

Following the treatment by Cohen et al.\ \cite{cohen19}, the randomized smoothing framework established a scalable and effective method for creating smoothed classifiers and certifying their robustness.  Salman et al.\ \cite{salman19} then developed an input manipulation attack on smoothed classifiers, using it for adversarial training to improve the provable robustness of these classifiers.  While the original randomized smoothing method was restricted to classifiers, the method has been extended in several directions, the extension to image segmentation by Fischer et al.\ \cite{fischer21} being most relevant to this work.

In \cite{tsipras19}, Tsipras et al.\ show that there exists a data distribution on which it is impossible to learn an accurate and robust classifier.  However, this is only a single and contrived example of such a phenomenon.  Additionally, Tsipras et al.\ note that in the limited data regine, as in this paper, adversarial training can be beneficial to a model's standard accuracy.

\section{Lane Line Instance Segmentation}
\label{sec:laneline}

The model we use for our experiments is a lane line instance segmentation model with an architecture adapted from Neven et al.\ \cite{neven18}, who themselves adopt the instance segmentation approach of DeBrabandere et al.\ \cite{debrabandere17}.  In \cite{debrabandere17}, an embedded feature space is learned from the data in a branch of the network having the same backbone architecture, but separate parameters from, a binary segmentation branch.  In contrast to \cite{debrabandere17}, who base their instance segmentation network on a DeepLabv1 architecture with a ResNet-38 backbone \cite{wu19}, we chose to adapt the DeepLabv3+ architecture \cite{chen18}, replacing DeepLabv3+'s Xception backbone with ResNet-101 \cite{he16}.  We made these choices based on the relative performance of different permutations of these models on other tasks of interest with similar datasets. As optimal model selection is not the subject of this work, we did not compare the performance of these permutations in adversarial conditions.

The binary segmentation branch of the network segments the pixels as belonging to either the background or a lane line, while the instance embedding branch embeds the pixels in a latent feature space.  Each pixel is embedded in this feature space in a manner so that the pixels for each instance are near each other, while pixels from distinct instances are far away from each other.  The loss function for the binary segmentation branch is the cross-entropy loss, while the instance embedding branch uses the discriminative loss function found in \cite{debrabandere17}.  The model's output is evaluated using the F-measure of the binary segmentation, as well as the symmetric best dice score \cite{scharr16} of the final instance segmentation.

\begin{figure}
\centering
\includegraphics[width=0.29\linewidth]{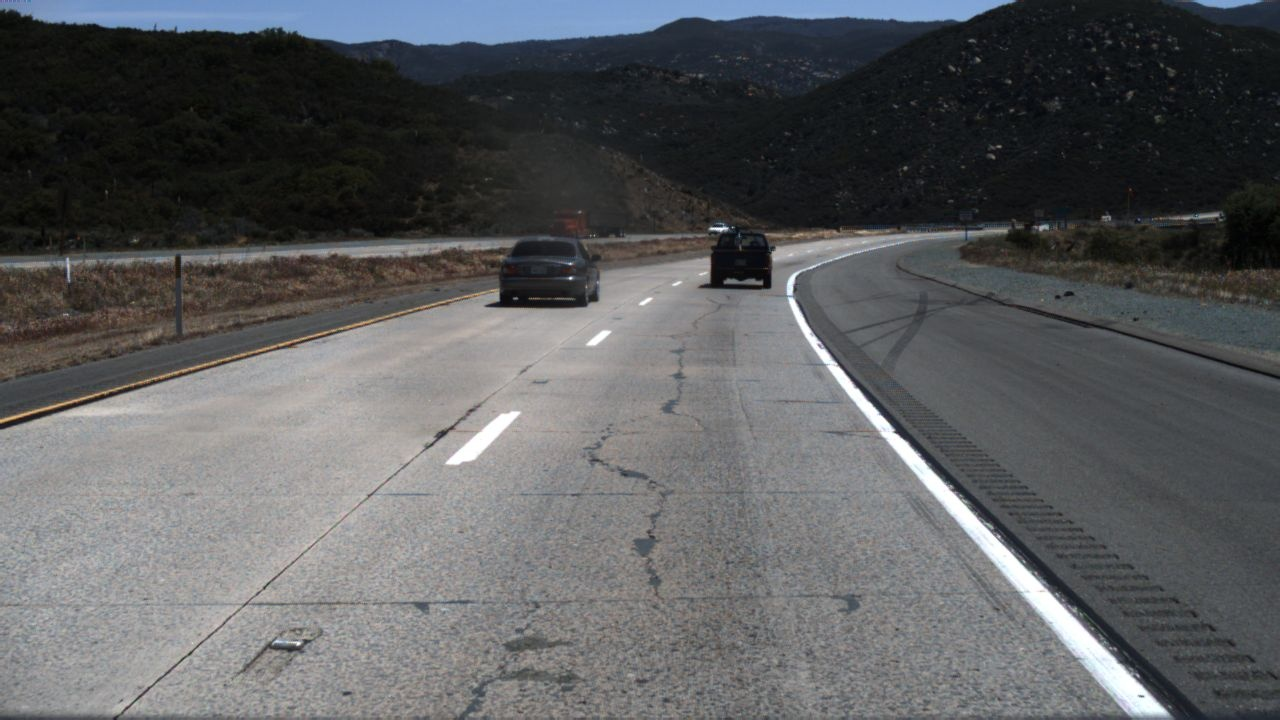}
\hspace*{0.04\linewidth}
\includegraphics[width=0.29\linewidth]{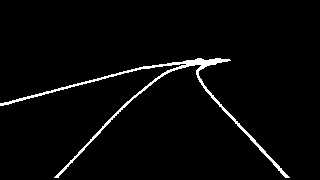}
\hspace*{0.04\linewidth}
\includegraphics[width=0.29\linewidth]{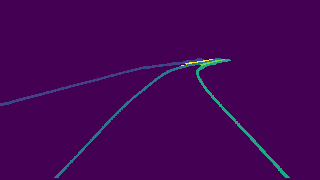}
\caption{A model input (left), the resulting binary segmentation (middle), and the final instance segmentation (right)}
\label{fig:lanenet}
\end{figure}

\section{Probabilistic Local Equivalence Certification}

Comparing the innate robustness of two models is challenging.  Evaluating the adversarial robustness of a model is often too aggressive, as any model which hasn't been adversarially trained will receive a score of 0.  Formally verifying a model typically requires an analysis limited to a specific architecture.  We propose a method based on randomiezed smoothing certification \cite{cohen19} which is applicable to arbitrary functions while also being more sensitive than evaluation of adversarial robustness.

\subsection{Probabilistic Local Equivalence Certification}

A function $f$ is robust if small changes to its input result in small changes to its output.  To this end, given inputs $x$ and $x'$, we will define a notion of how similar the outputs $f(x)$ and $f(x')$ are.  Let \(f: X \rightarrow Y\) be a function and let \(\mathcal{M}: Y \times Y \rightarrow [0, 1]\) be a scoring function evaluating how similar two elements \(y,y' \in Y\) are to each other.  Then \(\mathcal{M}\) induces a semipseudometric \(d\) on \(X\) by
\begin{equation}
d(x, x') = 1 - \mathcal{M}\big(f(x), f(x')\big)
\label{eqn:metric}
\end{equation}
and the set \(B_{d}(x, t) \coloneqq \{x' \in X \mid d(x, x') < t\}\) defines the points \(x' \in X\) so that \(f(x')\) is close to \(f(x)\) relative to \(\mathcal{M}\).  The robustness threshold $t$ is ultimately an application-specific choice depending on a number of factors, but the goal is to capture the amount of variation in the model's output which will avoid significant changes to any downstream results.  An extreme example, as mentioned below, occurs when $\mathcal{M}$ is the accuracy metric, in which case all values of $t \in (0, 1)$ are equivalent.

For \(d\) defined as in \eqref{eqn:metric}, the points in \(B_{d}(x, t)\) can be considered as being equivalent to \(x\) relative to \(f\), \(\mathcal{M}\), and \(t\) since for suitable $t$, $f$ gives equivalent outputs at all points of $B_{d}(x, t)$.  Returning to the concept of robustness, $f$ is robust if at all nearby inputs, $f$ gives similar outputs.  Translating this into the language introduced above, $f$ is robust at $x_{0}$ for radius $\delta$ if $B_{L^{2}}(x_{0}, \delta) \subseteq B_{d}(x_{0}, t)$.  While evaluating whether $B_{L^{2}}(x_{0}, \delta) \subseteq B_{d}(x_{0}, t)$ is difficult, the randomized smoothing framework provides a method for approximating this.  Specifically, the framework allows us to determine whether
\[\textrm{for all }x \in B_{L^{2}}(x_{0}, \delta),\quad P_{\epsilon \sim \mathcal{N}}\big(x + \varepsilon \in B_{d}(x_{0}, t) \big) > 0.5,\]
where $\mathcal{N}$ is a normal distribution.  In this direction, we have the following definitions.

\begin{defn}
Relative to a choice of \(d\), \(t\), and distribution $\mathcal{D}$,
\begin{enumerate}
\item If \(B_{L^{2}}(x, \delta) \subseteq B_{d}(x_{0}, t)\), then \(f\) is locally equivalent to \(f(x_{0})\) at \(x\) for radius \(\delta\).
\item If \(P_{\varepsilon \sim \mathcal{D}}\big(x + \varepsilon \in B_{d}(x_{0}, t)\big) \geq 0.5\), then \(f\) is probabilistically locally equivalent to \(f(x_{0})\) around \(x\).
\end{enumerate}
\end{defn}

Having defined \(d\) as in \eqref{eqn:metric}, fix a value of \(t\) and for any \(x \in X\), let \(B_{x} \coloneqq B_{d}(x, t)\).  The framework of \cite{cohen19,salman19} will be applied to the indicator function \(\mathds{1}_{x_{0}}: x \mapsto \mathds{1}(x \in B_{x_{0}})\).  Specifically, let $\mathcal{N} \coloneqq \mathcal{N}(0, \sigma^{2}I)$, the normal distribution centered at $0$ with standard deviation $\sigma$ with cumulative distribution function $\Phi$, and consider the smoothed function
\begin{equation}
\widehat{\mathds{1}_{x_{0}}}(x) \coloneqq \mathbb{E}_{\varepsilon \sim \mathcal{N}} \mathds{1}_{x_{0}}(x + \varepsilon) = P_{\varepsilon \sim \mathcal{N}}(x + \varepsilon \in B_{x_{0}}).
\label{eqn:smoothed}
\end{equation}
Then the \texttt{Certify} algorithm presented in \cite{cohen19}, when applied to \(\widehat{\mathds{1}_{x_{0}}}\), first determines a lower bound \(\underline{p}\) for \(\widehat{\mathds{1}_{x_{0}}}(x_{0})\), which establishes that \(f\) is locally equivalent to $f(x_{0})$ at \(x_{0}\) with probability at least \(\underline{p}\).  However, \cite{salman19} shows that \(\Phi^{-1} \circ \widehat{\mathds{1}_{x_{0}}}\) is \(1/\sigma\)-Lipschitz, so that the value \(\widehat{\mathds{1}_{x_{0}}}(x_{0})\) provides additional information about the surrounding neighborhood.  In this direction, the \texttt{Certify} algorithm applied to \(\widehat{\mathds{1}_{x_{0}}}\) is then able to provide a radius \(R\) guaranteeing, with probability \(1-\alpha\), that \(f\) is probabilistically locally equivalent to \(f(x_{0})\) around all \(x \in B_{L^{2}}(x_{0}, R)\).  This algorithm is reframed for the current context as \texttt{CertifyProbabilisticEquivalence}; see \cite{cohen19} for additional details surrounding the original algorithm.

\begin{algorithm}[ht]
    \caption{Certification of probabilistic local equivalence of a function \(f\)}
    \label{alg:certify}
\begin{algorithmic}[1]
    \Function{CertifyProbabilisticEquivalence}{\(f\), \(x_{0}\), \(y\), \(\sigma\), \(n\), \(\alpha\), \(\mathcal{M}\), t}
    \State \(y_{0} \gets f(x_{0})\)
    \State \(x' \gets\) \(n\) samples from \(\mathcal{N}(x, \sigma^{2}I)\)
    \State \(y' \gets f(x')\)
    \If{\(y\) is \texttt{None}}
        \State \texttt{num\_equivalent} \(\gets\) C\textsc{ount}T\textsc{rue}\(\big(1 - \mathcal{M}(y_{0}, y') < t\big)\)
    \Else{}
	  \State \texttt{num\_equivalent} \(\gets\) C\textsc{ount}T\textsc{rue}\(\Big(1 - \frac{\mathcal{M}(y_{0}, y)}{\mathcal{M}(y', y)} < t\Big)\)
    \EndIf
    \State \(\underline{p} \gets\) L\textsc{ower}C\textsc{onf}B\textsc{ound}\((\texttt{num\_equivalent}, n, 1-\alpha)\)
    \If{\(\underline{p} > \frac{1}{2}\)}\Return{radius \(\sigma\Phi^{-1}(\underline{p})\)}
    \Else{} \Return{ABSTAIN}
    \EndIf
\end{algorithmic}
\end{algorithm}

\begin{algorithm}[ht]
    \caption{Compute a robustness score of the function \(f\)}
    \label{alg:score}
\begin{algorithmic}[1]
    \Function{RobustnessScore}{\(f\), \(E\), \(\sigma\), \(n\), \(\alpha\), \(\mathcal{M}\), \(t\)}
    \State \texttt{radii} \(\gets\) \([]\)
    \For{\((x, y) \in E\)}
        \State \texttt{radius} \(\gets\) C\textsc{ertify}P\textsc{robabilistic}E\textsc{quivalence}\((f, x, y, \sigma, n, \alpha, \mathcal{M}, t)\)
        \If{\texttt{radius} is ABSTAIN}\textbf{then} Append \(0\) to \texttt{radii}
        \Else{} Append \texttt{radius} to \texttt{radii}
        \EndIf
    \EndFor
    \State \Return M\textsc{ean}(\texttt{radii})
\end{algorithmic}
\end{algorithm}

\begin{defn}
Let
\begin{multline*}
\mathrm{ProbLocEquiv}(S, f, r) \coloneqq
\{x_{0} \in S \mid\\
f\textrm{ is probabilisitically locally equivalent to }f(x_{0})\textrm{ around all }x \in B_{L^{2}}(x_{0}, r)\}.
\end{multline*}
\end{defn}

\begin{prop}
With probability at least \(1 - \alpha\) over the randomness in \verb!Certify! \verb!ProbabilisticEquivalence!, if \verb!CertifyProbabilisticEquivalence! returns a radius \(R\), then \(x \in \mathrm{ProbLocEquiv}(X, f, R)\).
\label{prop:certify}
\end{prop}

For a given radius \(r\) and evaluation set \(E\), Proposition \ref{prop:certify} can be used to determine the percentage of \(E\) which lie in \(\mathrm{ProbLocEquiv}(E, f, r)\), giving an indication of the robustness of the model to perturbations with an \(L^{2}\) norm less than \(r\).  A robustness score, relative to a given evaluation set $E$ and choice of hyperparameters, can then be computed for a function $f$ by taking the mean of the certified radii over the set $E$.  As in the case of randomized smoothing certification, labels for the evaluation set $E$ are not required.

\subsection{When Labels Are Available}

In the case that ground-truth labels $\{y_{i}\}$ are available for points $\{x_{i}\} \subseteq X$, there is an alternative formulation for \eqref{eqn:metric}.  In this case, rather than simply looking at the similarity between \(f(x_{1})\) and \(f(x_{2})\), it may be desirable to determine how these values differ relative to the ground-truth \(y\).  When a pair \((x, y) \sim \mathcal{D} = X \times Y\), \eqref{eqn:metric} can be replaced with
\begin{equation}
d_{y}(x, x') = 1 - \frac{\mathcal{M}(f(x'), y)}{\mathcal{M}(f(x), y)}.
\label{eqn:metric2}
\end{equation}
The convention $d_{y}(x, x') = 1$ if $\mathcal{M}(f(x), y) = 0 = \mathcal{M}(f(x'), y)$ and $d_{y}(x, x') = -\infty$ if $\mathcal{M}(f(x), y) = 0 \not= \mathcal{M}(f(x'), y)$ will be adopted.
While \(d_{y}\) is no longer even a semipseudometric, the set \(B_{d_{y}}(x, t)\) contains the points \(x'\) so that replacing \(x\) by \(x'\) does not significantly degrade the model's performance.  In particular, note that if \(\mathcal{M}(f(x'), y) > \mathcal{M}(f(x), y)\), then \(d_{y}(x, x') < 0\) and \(x' \in B_{d_{y}}(x, t)\).

\subsection{The Case of Classification}

Importantly, in the case of classifiers, the above measure of robustness reduces to the standard measure of robustness.  In this case, \(\mathcal{M}\) is the accuracy metric:
\[\mathcal{M}(y_{1}, y_{2}) = \begin{cases}
1 & \textrm{if \(y_{1} = y_{2}\),}\\
0 & \textrm{otherwise.}
\end{cases}\]
Defining \(d\) as in \eqref{eqn:metric},
\begin{align*}
x' \in B_{d}(x, t) &\Longleftrightarrow 1 - \mathcal{M}\big(f(x), f(x')\big) < t
\Longleftrightarrow 1 - t < \mathcal{M}\big(f(x), f(x')\big),
\end{align*}
so that for \(0 < t < 1\), \(x' \in B_{d}(x, t) \Longleftrightarrow f(x) = f(x')\).

\subsection{The Case of Lane Line Segmentation}
\label{sec:lanelineequiv}

For lane line instance segmentation evaluation, we use symmetric best dice (SBD) as the score function \(\mathcal{M}\) and define \(d\) as in \eqref{eqn:metric2}.  Figures \ref{fig:lanesegment} and \ref{fig:lanesegmentnoisy} illustrate the situation using a model trained following the approach in Section \ref{sec:advtrain}, showing an equivalent and an inequivalent example.  Figure \ref{fig:lanesegment} shows the original input, the model's output, and the ground truth label, leading to an SBD score of \(0.674\).  Using a relative threshold of \(t = 0.1\) and rearranging \eqref{eqn:metric2} results in
\[d_{y}(x, x') < 0.1 \Longleftrightarrow 0.9 \mathcal{M}(f(x), y) < \mathcal{M}(f(x'), y) \Longleftrightarrow \mathcal{M}(f(x'), y) > 0.60066.\]
In Figure \ref{fig:lanesegmentnoisy}, the first example, using \(\sigma = 0.2\), leads to an SBD score of \(0.661\), indicating that the model \(f\) has made an equivalent prediction on the given input.  However, the second example, using \(\sigma = 0.3\), results in an SBD score \(0.573\), outside of the threshold considered as equivalent.

\begin{figure}
\centering
\includegraphics[width=3cm]{images/orig.png}
\hspace*{1cm}
\includegraphics[width=3cm]{images/pred.png}
\hspace*{1cm}
\includegraphics[width=3cm]{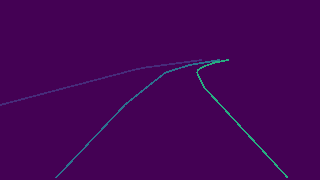}
\caption{The original input \(x\), the model's prediction \(y\), and the ground truth label, resulting in a symmetric best dice score of \(0.674\).}
\label{fig:lanesegment}
\end{figure}

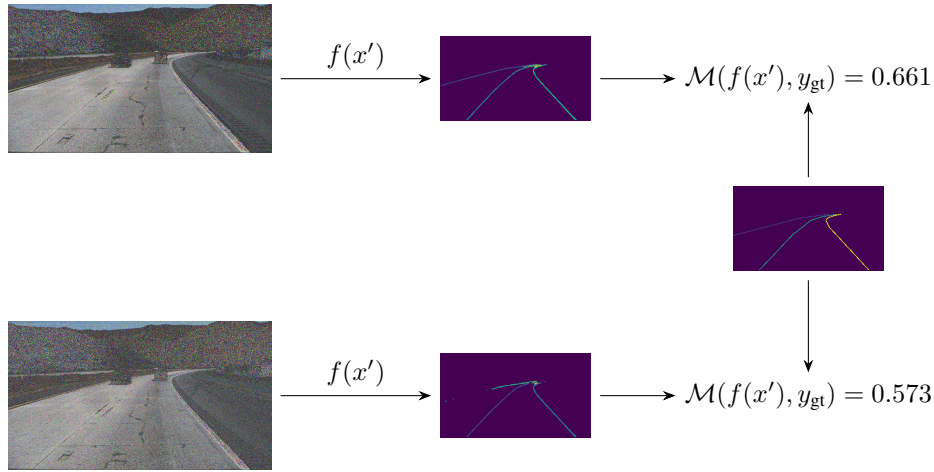
\begin{figure}
\centering
\begin{tikzpicture}
\node (TL) {\includegraphics[width=3.5cm]{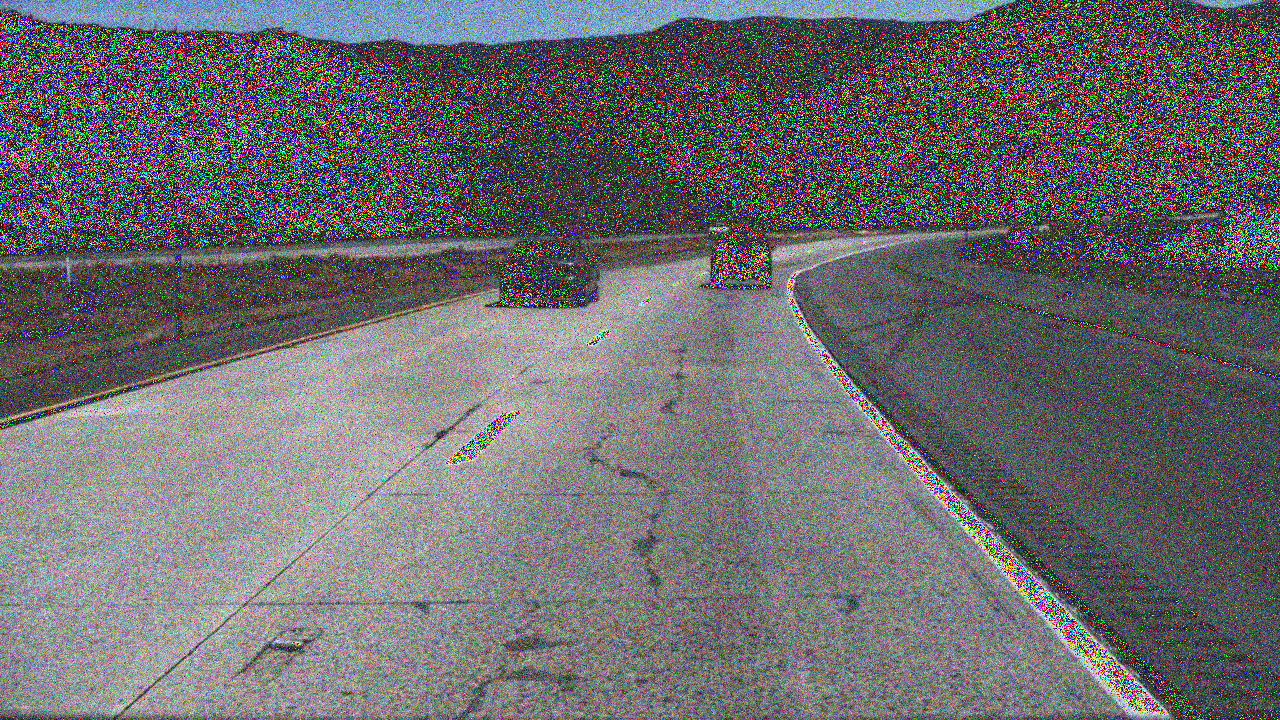}};
\node[right=2cm of TL] (TR) {\includegraphics[width=2cm]{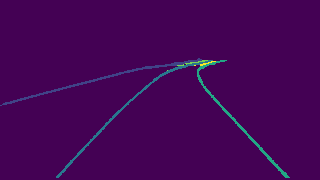}};
\node[below=2cm of TL] (BL) {\includegraphics[width=3.5cm]{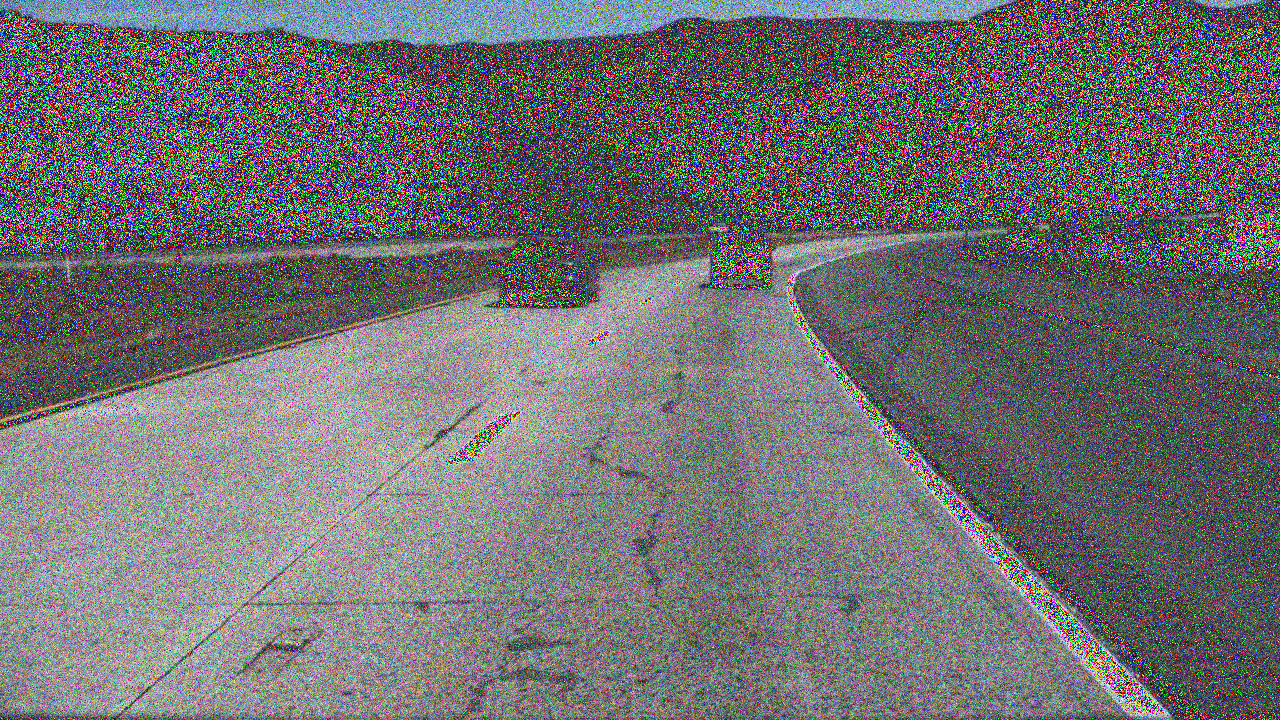}};
\node[right=2cm of BL] (BR) {\includegraphics[width=2cm]{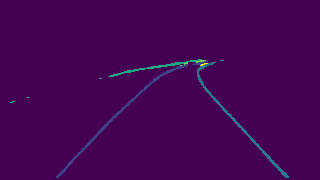}};
\node[right=1cm of TR] (TRR) {\(\mathcal{M}(f(x'), y_{\textrm{gt}}) = 0.661\)};
\node[right=1cm of BR] (BRR) {\(\mathcal{M}(f(x'), y_{\textrm{gt}}) = 0.573\)};
\node[below=1cm of TRR] (MR) {\includegraphics[width=2cm]{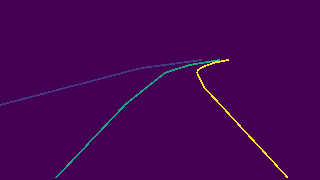}};

\path[-{Stealth[]}]
  (TL) edge node[midway,above] {\(f(x')\)} (TR)
  (BL) edge node[midway,above] {\(f(x')\)} (BR)
  (TR) edge (TRR)
  (BR) edge (BRR)
  (MR) edge (TRR)
  (MR) edge (BRR);
\end{tikzpicture}
\caption{(Above) An input leading to a symmetric best dice score of \(0.661\), within the defined permissible threshold of \(0.0674\) from the original score \(0.674\). (Below) An input leading to a symmetric best dice score of \(0.573\), outside of the permissible threshold.}
\label{fig:lanesegmentnoisy}
\vspace{-\baselineskip}
\end{figure}

\section{Adversarial Training Image Segmentation}
\label{sec:advtrain}

While the explicit premise of adversarial training \cite{madry18} is to improve robustness of a network to small perturbations through data augmentation, the result is improved robustness beyond just the small perturbations included in the data augmentation.  As noted in Engstrom et al.\ \cite{engstrom19}, as human perception is invariant to these perturbations, any model intended to mimic human perception should be invariant to them as well.  In this way, training a model to be robust to these perturbations serves to improve the model's alignment with human perception.  Further, because the model is able to identify features which are invariant to the worst-case perturbations presented during adversarial training, the model should also be robust to other non-worst-case perturbations.  As such, adversarial robustness is a desirable property for models to possess even beyond its importance for security considerations.

A signifcant improvement toward the use of adversarial training for production systems was made by Wong et al.\ \cite{wong20}, reducing the computation required for adversarial training, accomplishing this by leveraging perturbations generated using FGSM after initializing the perturbation at a random point within \(B_{\epsilon}(0)\).  Computing an update then requires only two backpropogation computations:  one to compute the manipulated input and one to compute the weight update.  The result is that adversarial training becomes tractable, even for relatively large models and models with production deadlines.

However, there remains a question of the performance of adversarially trained models on the original training data distribution.  Given the entire training pipeline is optimized towards performing well on training distributions, it is unsurprising that adversarially-trained models within the same pipeline often struggle to match this performance when being asked to {\it also} perform well on a manipulated training distribution.  On the other hand, as the inference-time distribution often varies from the training distribution, there is evidence that the adversarially-trained model may outperform standardly-trained models in the presence of this distribution shift \cite{salman20}.

Nevertheless, improvements to the base adversarial training procedure have helped to improve both the standard and adversarial performance of adversarially-trained models.  One of the more common approaches to improving performance is the use of the TRADES loss function \cite{zhang19}.  Zhang et al.\ adds a regularization term to improve the robustness of the model by seeking to maximize the distance of the decision boundary from each input.  This is accomplished by minimizing the difference between the model's output on an original input and a resulting manipulated input.

Because TRADES is more readily adapted to the binary segmentation output of our network, we use FBF as our base training method to provide robustness to the instance embedding branch.  We then add the TRADES robust loss for the binary segmentation branch.  Note that the original input \(x'\) received by the TRADES regularization has already been manipulated from the input \(x\) in the training set by FBF using FGSM.  That input is then further manipulated by the TRADES regularization to obtain \(x''\), in this case to maximize the KL divergence between \(f(x'')\) and \(f(x')\).  The total loss is then a linear combination of the model's standard loss and the TRADES regularization term.
Additional details are given in Algorithms \ref{alg:train} and \ref{alg:loss}.  For information related to network-specific outputs or loss functions, see Section \ref{sec:laneline} and the corresponding references.

\section{Experiments}

\subsection{Compute Resources Used for Experiments}
\label{subsec:compute}

Experiments were performed on an internal Kubernetes cluster containing NVIDIA DGX-1 machines.  Each node on the cluster contains 8\(\times\) NVIDIA Tesla V100 GPUs, 2\(\times\) 20-Core Intel Xeon%

\begin{algorithm}[ht]
    \caption{FBF + TRADES training of an instance segmentation model \(f\) with parameters \(\vartheta\) for \(T\) epochs, given perturbation bound \(\varepsilon\), step size \(\alpha\), and a dataset of size \(M\)}
    \label{alg:train}
\begin{algorithmic}[1]
    \Function{Train}{\(f\), \(\vartheta\), \(X\), \(Y\)}
    \For{\(t = 1, \ldots, T\)}
    \For{\(i = 1, \ldots, M\)}
        \State \(\delta \gets \mathrm{Uniform}(-\varepsilon, \varepsilon)\)
        \State \(\delta \gets \Pi_{B_{L^{\infty}}(x_{i}, \varepsilon)} \:\delta + \alpha \cdot \sgn(\nabla_{\delta}\)L\textsc{oss}\((f(x_{i} + \delta), y_{i}))\)
        \State \(\vartheta \gets \vartheta - \nabla_{\vartheta}\)L\textsc{oss}\((f(x_{i} + \delta), y_{i})\)
    \EndFor
    \EndFor
    \State \Return \(\vartheta\)
\end{algorithmic}
\end{algorithm}

\begin{algorithm}[ht]
    \caption{Instance segmentation model loss with TRADES regularization for \(N\) steps, given step size \(\eta\), perturbation bound \(\varepsilon\), and regularization weight \(\beta\)}
    \label{alg:loss}
\begin{algorithmic}[1]
    \Function{Loss}{\(f\), \(x\), \(y_{\textrm{seg}}\), \(y_{\textrm{inst}}\)}
    \State \((\texttt{segmentation}, \texttt{instance\_embedding}) \gets f(x)\)
    \State \(\texttt{seg\_loss} \gets\)C\textsc{ross}E\textsc{ntropy}L\textsc{oss}\((\texttt{segmentation}, y_{\textrm{seg}})\)
    \State \(\texttt{instance\_loss} \gets\)D\textsc{iscriminative}L\textsc{oss}\((\texttt{instance\_embedding}, y_\textrm{inst})\)
    \State //Compute input for TRADES robust loss:
    \State \(x' \gets x + 0.001 \cdot \mathcal{N}(0, I)\)
    \For{\(j = 1, \ldots, N\)}
        \State \((\texttt{segmentation}', \texttt{instance\_embedding}') \gets f(x')\)
        \State \(x' \gets \Pi_{B_{L^{\infty}}(x, \varepsilon)} \:x' + \eta \sgn \nabla_{x'}\)KLD\textsc{iv}L\textsc{oss}\(\big(\texttt{segmentation}, \texttt{segmentation}'\big)\)
    \EndFor
    \State \texttt{trades\_loss}\( \gets \)KLD\textsc{iv}L\textsc{oss}\(\big(\texttt{segmentation}, \texttt{segmentation}'\big)\)
    \State \Return \(\texttt{seg\_loss} + \texttt{instance\_loss} + \beta \cdot\texttt{trades\_loss}\)
\end{algorithmic}
\end{algorithm}

E5-2698 v4 CPUs at a clock speed of 2.2 GHz for a total of 80 logical cores, and 512 GB of DDR4 RAM at a bus speed of 2.133 GHz.  Each individual experiment conducted in this paper requested one of the Tesla V100 GPUs, 9 CPU cores, and 60 GB of RAM.

\subsection{Carbon Emissions Related to Experiments}

Altogether, training and validation of our final experiments, those displayed in Figure \ref{fig:performance}, totaled 7037 hours of GPU utilization (Appendix \ref{app:training}).  Based on the utilized GPU, experiment duration, and Michigan's carbon efficiency of \SI{0.4976}{kg.CO\textsubscript{2}/kWh} \cite{miprofile}, the ML CO\textsubscript{2} Impact Calculator \cite{mlco2impact}, as presented in \cite{lacoste19}, estimates the carbon emissions from these experiments to be \SI{1050.49}{kg.CO\textsubscript{2}}.  Additionally, testing the best model from each training method and computing the security and robustness curves in Figure \ref{fig:robustness} required 101.84 hours of computation on an NVIDIA Tesla P100 GPU, resulting in additional emissions of \SI{12.67}{kg.CO\textsubscript{2}}.  According to \cite{epamiles}, our total emissions equates to \SI{4298}{km} driven by an average internal combustion engine vehicle (\SI{0.247}{kg.CO\textsubscript{2}/km}) or \SI{13819}{km} driven by an average electric vehicle (\SI{0.077}{kg.CO\textsubscript{2}/km}; Appendix \ref{app:carbon}).

\subsection{Adversarial Training Setup}
\label{subsec:setup}

Training was performed using the implementation of FBF adversarial training in \cite{art}, modified for use with non-classifiers, using inputs in the space \([-1, 1]^{3\times720\times1280}\).  The dataset used for development was the TuSimple lane detection challenge \cite{tusimple17}, with a train/validation/test split of 3082/181/363 images.  We used the SGD optimizer with a momentum factor of \(0.9\), a perturbation bound of \(\varepsilon = 8/255\) for both FBF and TRADES, \(N = 10\) steps of size \(\eta = \varepsilon/10\) for TRADES, and a regularization weighting of \(\beta = 2.0\times10^{-5}\) that was tuned through experimentation.  The TRADES repository \cite{trades} recommends \(1 \leq \beta \leq 10\) for training a classifier for which the KL divergence is calculated over a one-dimensional array of class probabilities.  As our KL divergence was calculated over a \(320\times180\)-dimensional array of probabilities, this translates to a recommendation of \(1.74\times10^{-5} \leq \beta \leq 17.4\times10^{-5}\), in line with our choice of \(\beta\).  We used \(L^{\infty}\)-norm perturbations for TRADES, corresponding to the use of FGSM in FBF adversarial training.

Because many aspects of AutoAttack \cite{croce20} are targeted at classification, for adversarial evaluation we generated manipulated inputs using the untargeted Projected Gradient Descent attack with a perturbation bound of \(\varepsilon = 8/255\), a step size of \(2/255\), \(20\) steps of gradient descent, and an \(L^{\infty}\) distance metric.  For the evaluation phase of the FBF+TRADES model, the TRADES loss was not included in the computation of the manipulated inputs.

\subsection{Results}

To balance the competing objectives of maximizing both the F-measure of the binary segmentation result as well as the symmetric best dice (SBD) score of the instance segmentation result, we adopt a weighted sum approach to this multi-objective optimization problem, assigning a weight of 1 to each objective function.  The resulting objective, evaluated on the validation set, is plotted as a function of the number of training epochs in Figure \ref{fig:performance}.  Plots for the individual metrics are provided in Appendix \ref{app:performance}.  In the case of a manipulated input, the typical behavior of PGD was to cause the model to classify the entire image as a single lane line, resulting in an F-measure of \(\approx 0\).  Due to the symmetric nature of SBD, this lane line class was associated with the background class of the ground-truth label, rather than any of the \(3\) lane line instances, resulting in an SBD score of \(\approx 1/4\).  This leads to a minimum adversarial \(F_{1} + \textrm{SBD}\) score of \(\approx 0.25\).

\begin{figure}[ht]
    \centering
    \begin{subfigure}{0.46\textwidth}
        \centering
        \includegraphics[width=\textwidth]{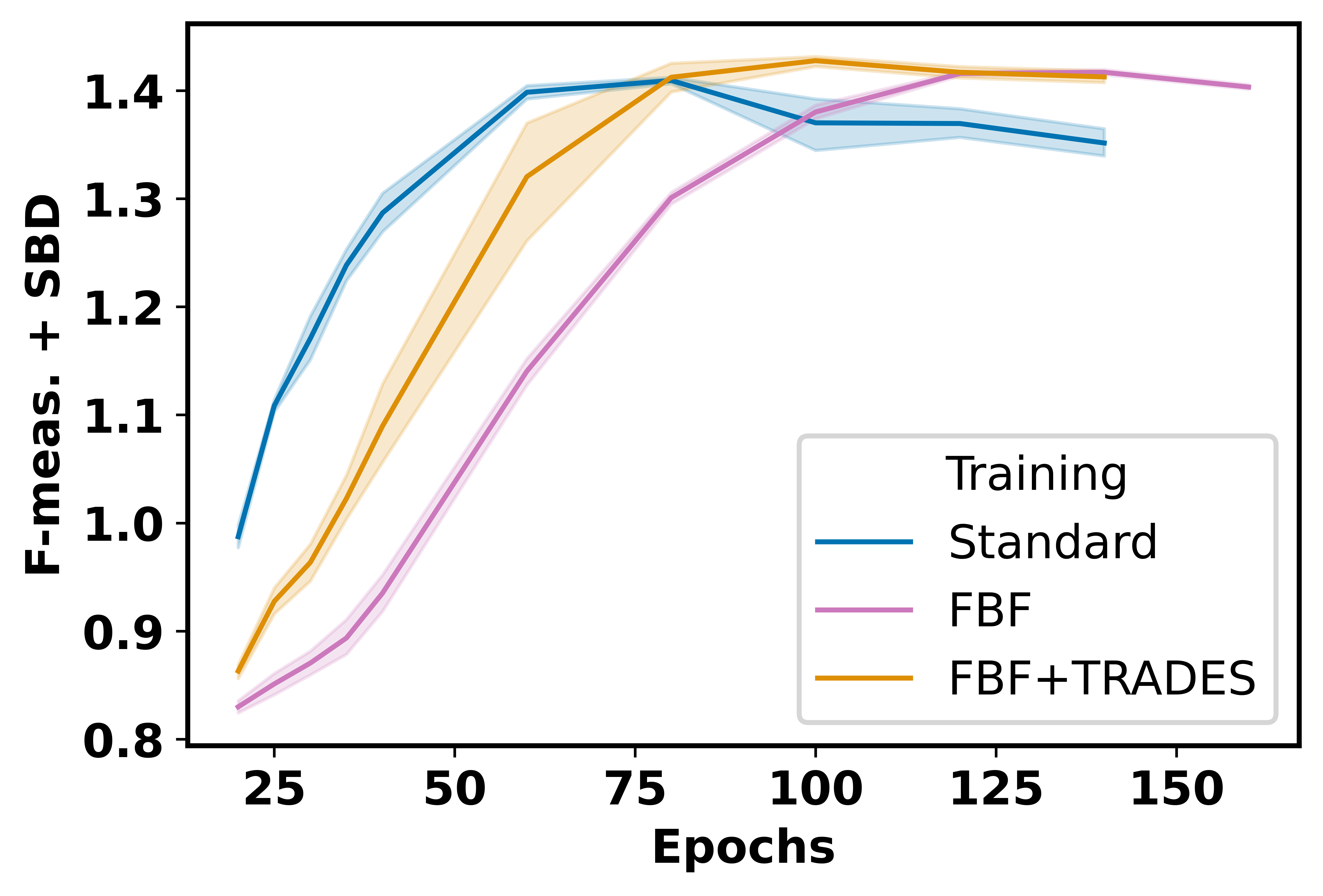}
    \end{subfigure}
    \hfill
    \begin{subfigure}{0.46\textwidth}
	  \centering
        \includegraphics[width=\textwidth]{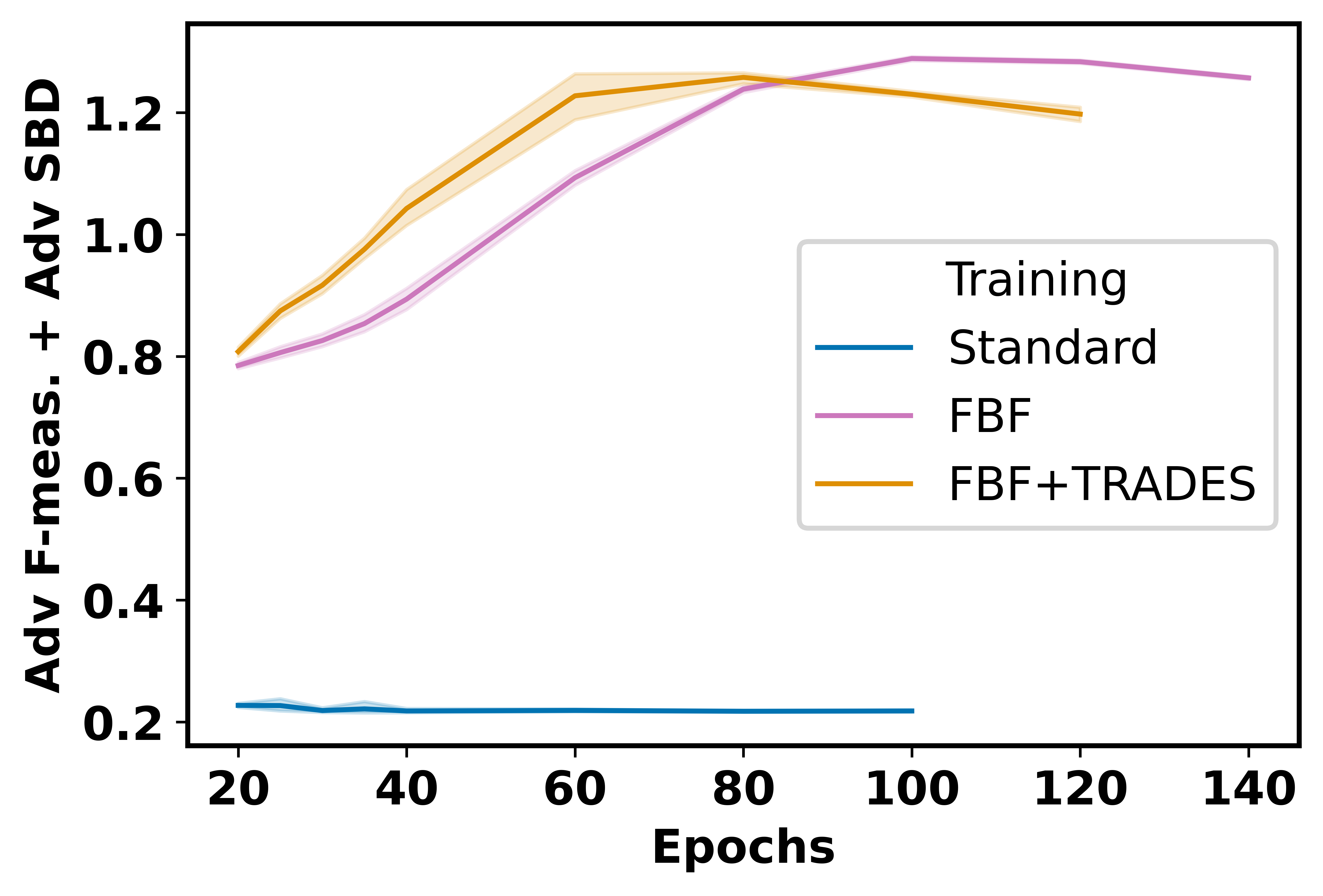}
    \end{subfigure}
    \caption{Performance of the models on natural (left) and adversarial (right) inputs from the validation set as a function of the number of training epochs}
    \label{fig:performance}
\end{figure}

Ultimately, we found that while standard training was able to achieve a peak in performance more rapidly than FBF+TRADES or FBF, both FBF+TRADES and FBF were able to match this performance with sufficient training epochs.  Notably, the performance on natural inputs of the best adversarially-trained models did not suffer as a result of adversarial training, in contrast to what is frequently observed when adversarially training on large object classification datasets.  Both FBF+TRADES and FBF achieved similar levels of robustness, compared to no robustness from standard training, through FBF+TRADES was able to reach a peak in both natural and adversarial performance more quickly than FBF.  Additionally, we found that after peaking, the F-measure of the models plateaued with additional training, while the SBD score began to drop as the models began overfitting to the training set (Appendix \ref{app:performance}).

\begin{table}[h]
\centering
\caption{Performance of best model from each training method on the test set}
\begin{tabular}{lccccc}
\toprule
& \multicolumn{2}{c}{Natural} & \multicolumn{2}{c}{Adversarial} & \\
\cmidrule(lr){2-3}\cmidrule(lr){4-5}
Training Method & F-measure & SBD & F-measure & SBD & Time (hrs)\\
\midrule
Standard (80 Epochs) & 0.59 & 0.82 & 0.00 & 0.21 & 7.28\\
FBF (\cite{wong20}; 140 Epochs) & 0.58 & 0.84 & 0.49 & 0.76 & 27.01\\
FBF+TRADES (Ours; 100 Epochs) & 0.59 & 0.85 & 0.48 & 0.76 & 22.78\\
\bottomrule
\end{tabular}
\label{tab:performance}
\end{table}

For each training method, we selected the model which performed best on the validation set and listed its performance on the test set in Table \ref{tab:performance}.  To better understand the robustness of the different models, we performed a more in-depth analysis of the robustness using the test set.  We first performed an analysis of the models' local equivalence robustness following the method in Section \ref{sec:lanelineequiv} and plotted the results in Figure \ref{fig:robustcurve}, finding that adversarial training improves the robustness of the model to the random perturbations generated in the evaluation.  The local equivalence robustness analysis was performed using \(\sigma = 0.1\), \(\alpha = 0.05\), and \(n = 160\) samples per input.  Figure \ref{fig:robustcurve} shows that local equivalence robustness is able to distinguish between the levels of robustness to normally distributed perturbations present in the standardly-trained and adversarially-trained models, and even between the different adversarially-trained models.  We leave to a future work an analysis of how other robust networks, such as those robust to common corruptions \cite{hendrycks20}, perform under this evaluation.

We also evaluated the models' adversarial robustness.  Although, we weren't able to apply AutoAttack \cite{croce20}, to mitigate the effect of the perturbation budget $\varepsilon$ as a hyperparameter, we plotted the security curve for each model, evaluating the model's performance under a 10-step \(L^{\infty}\) PGD manipulation across a range of perturbation budgets in Figure \ref{fig:seccurve}.  As noted in \cite{croce20}, after a few iterations of PGD, the loss function plateaus (see Appendix \ref{app:steps}), so that 10-step PGD results in a good approximation of the optimum, while balancing the computational load required for a thorough sweep of perturbation budgets.  For a given budget \(\varepsilon\), a step size of \(3\varepsilon/(2\cdot10)\) was used.

\begin{figure}[ht]
    \centering
    \begin{subfigure}[b]{0.46\textwidth}
        \centering
        \includegraphics[width=\textwidth]{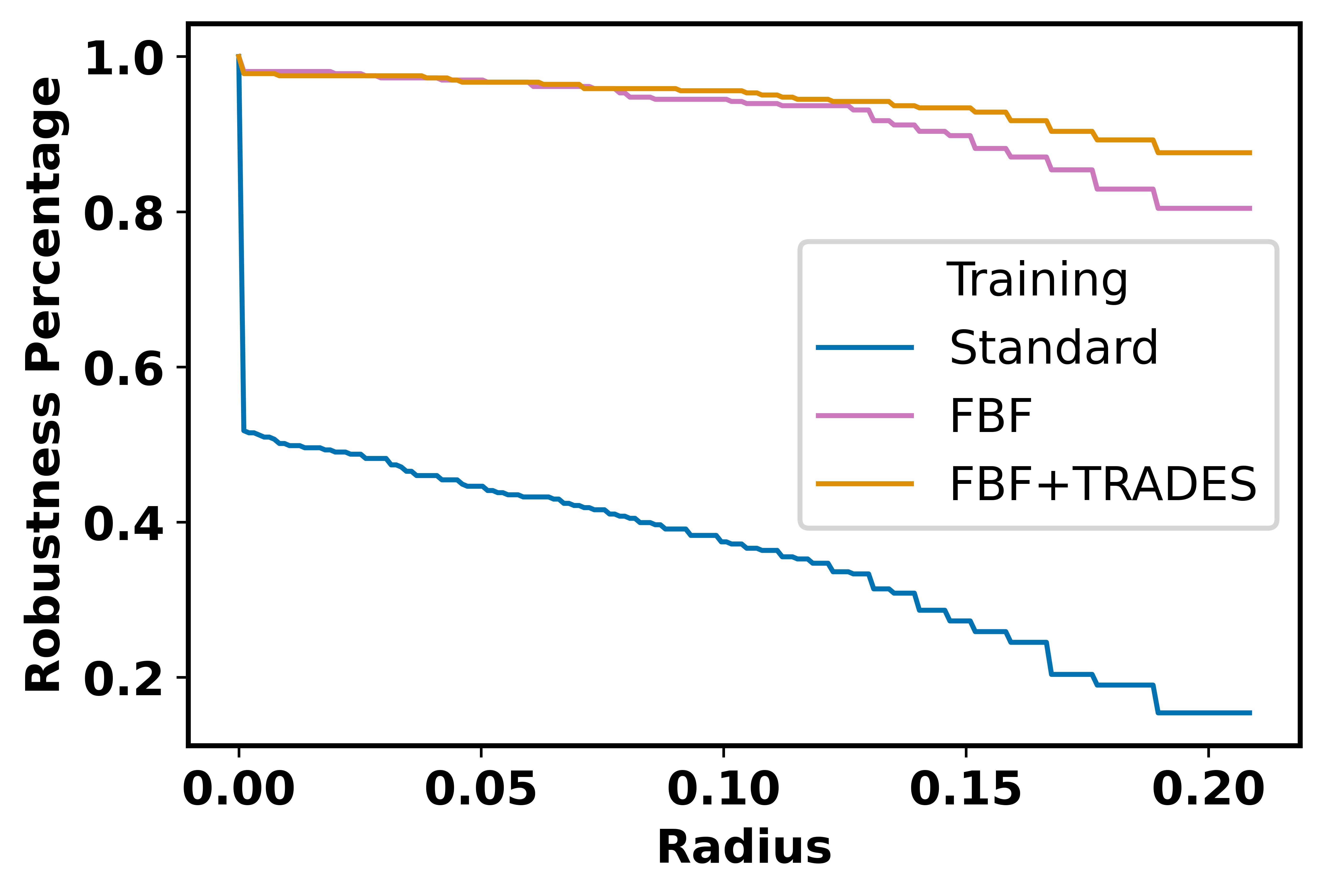}
        \caption{Robustness curve for each training method  (\(\sigma = 0.1, n = 160\))}
	  \label{fig:robustcurve}
    \end{subfigure}
    \hfill
    \begin{subfigure}[b]{0.46\textwidth}
        \centering
        \includegraphics[width=\textwidth]{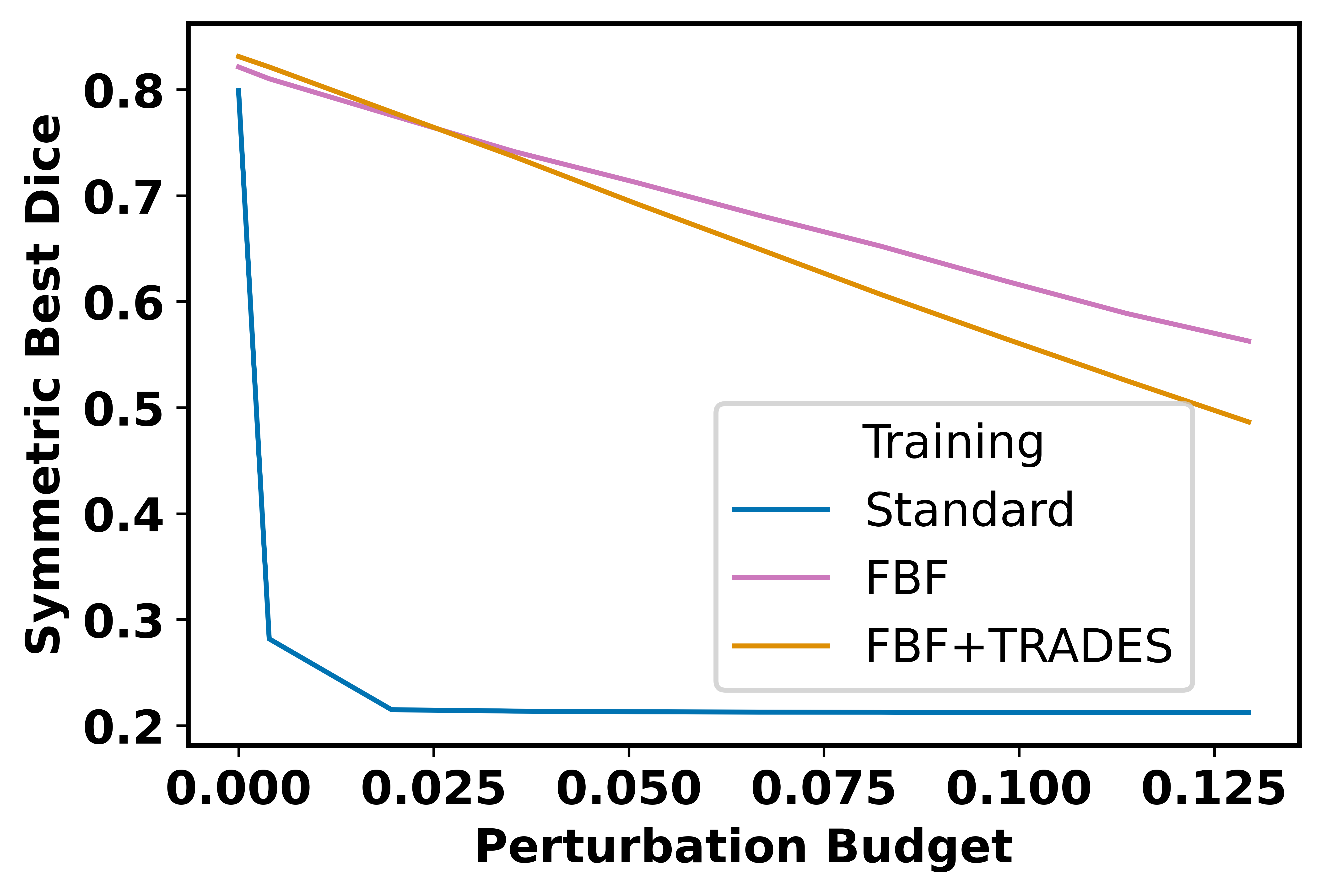}
        \caption{\(L^{\infty}\) security curve for each training method (10 steps of PGD)}
        \label{fig:seccurve}
    \end{subfigure}
    \caption{Robustness evaluation of the best performing model from each training method}
    \label{fig:robustness}
\end{figure}


From Figure \ref{fig:robustness}
, we see that while FBF training slightly outperforms FBF+TRADES training in robustness to \(L^{\infty}\) PGD attacks with budget \(\varepsilon > 8/255\), FBF+TRADES provides improved robustness to the random perturbations used in the probabilistic local equivalence analysis compared to FBF.  The security curve for \(L^{2}\) PGD 
and the local equivalence curve for \(\sigma = 0.25\) are given in Appendix \ref{app:l2robustness}.

\section{Conclusion}

We introduced a method based on the theory of randomized smoothing \cite{cohen19} for assessing the robustness of an arbitrary function to randomized input perturbations.  We showed probabilisitic local equivalence is able to discriminate between a standardly-trained model and an adversarially-trained model.  While the formulation of adversarial training implies robustness to such random perturbations, probabilistic local equivalence demonstrates that the standardly-trained model does not exhibit robustness to such random perturbations and that adversarial training is able to significantly improve the robustness of a model to this type of perturbation as well.

Further, we demonstrated that, in certain instances, adversarial training, whether FBF or FBF+TRADES, is able to match the performance of standard training on natural inputs.  It is likely that the size of the network (\(2\times\) ResNet-101 backbones) relative to the size of the training dataset (3082 images) played a role in these results.  Nevertheless, our results show that the foregone assumption that adversarial training will necessarily reduce the performance of the network on the training data distribution does not hold in general.  Rather, it stands to reason that the reduction in performance on the training distribution is most prevalent when the network is overfitting to the that distribution.  In our case, where the size of the dataset relative to the model makes it likely the model will underfit, adversarial training served as a form of data augmentation improving the network's performance on the training distribution and beyond.  Alternatively, in the regime where overfitting is likely, adversarial training can help to prevent such overfitting to the training distribution.

\bibliographystyle{abbrv}
\bibliography{robust_eval}

\appendix

\section{Proof of Proposition 1}
\label{app:proof}

\begin{lem}
\cite[Appendix A]{salman19}
Let \(f: \mathbb{R}^{n} \rightarrow [0, 1]^{2}\) be defined by \(f(x) = \big(f_{0}(x), f_{1}(x)\big)\) with \(f_{i}: \mathbb{R}^{n} \rightarrow [0, 1]\) and \(f_{0}(x) + f_{1}(x) = 1\) for all \(x \in \mathbb{R}^{n}\).  Let \(\varepsilon \sim \mathcal{N}(0, \sigma^{2}I)\).  If \(\mathbb{E}_{\varepsilon} f_{1}(x + \varepsilon) \geq \frac{1}{2}\), then
\[\mathbb{E}_{\varepsilon} f_{1}(x + \delta + \varepsilon) \geq \frac{1}{2}\textrm{ for all }\lvert \delta \rvert_{2} < \sigma \Phi^{-1} \circ f_{1}(x).\]
\label{lem}
\end{lem}

\begin{proof}
Appendix A in \cite{salman19} shows that the lemma holds for all \(\delta\) with
\[\lvert \delta \rvert_{2} \geq \frac{1}{2}\Big(\Phi^{-1} \circ f_{1}(x) - \Phi^{-1} \circ f_{0}(x)\Big).\]
Noting that \(1 - \Phi \circ \Phi^{-1} \circ f_{1}(x) = \Phi \circ \Phi^{-1}\big(-f_{1}(x)\big)\),
\begin{align*}
\Phi^{-1}\circ f_{0}(x) = \Phi^{-1}\Big(1-f_{1}(x)\Big) &= \Phi^{-1}\Big(1 - \Phi \circ \Phi^{-1} \circ f_{1}(x)\Big)\\
&= \Phi^{-1} \circ \Phi\Big(-\Phi^{-1} \circ f_{1}(x)\Big) = -\Phi^{-1} \circ f_{1}(x),
\end{align*}
and the result follows.
\end{proof}

\begin{prop1}
For any function \(f: \mathbb{R}^{n} \rightarrow [0, 1]\) and \(x \in X \subseteq \mathbb{R}^{n}\), with probability at least \(1 - \alpha\) over the randomness in \verb!CertifyProbabilisticEquivalence!, if \verb!CertifyProbabilisticEquivalence! returns a radius \(R\), then \(x \in \mathrm{ProbLocEquiv}(X, f, R)\).
\end{prop1}

\begin{proof}
If \texttt{CertifyProbabilisticEquivalence} returns a radius \(R\), then \verb!LowerConf!\verb!Bound! returned a value \(\underline{p} = \Phi(R/\sigma) > 1/2\) and as a lower bound on the probability of success in the Bernoulli trial of being locally equivalent,
\begin{align*}
\underline{p} &\leq P_{y'}\big(1 - \mathcal{M}(y_{0}, y')\big) < t)\\
&= P_{x'}\Big(1 - \mathcal{M}\big(f(x_{0}), f(x')\big) < t\Big)\\
&= P_{\varepsilon \sim \mathcal{N}}\Big(1 - \mathcal{M}\big(f(x_{0}), f(x_{0}+\varepsilon)\big) < t\Big)\\
&= P_{\varepsilon \sim \mathcal{N}}(d(x_{0}, x_{0} + \varepsilon) < t) \qquad\textrm{(from \eqref{eqn:metric})}\\
&= P_{\varepsilon \sim \mathcal{N}}\big(x_{0} + \varepsilon \in B_{d}(x_{0}, t)\big)\\
&= \widehat{\mathds{1}_{x_{0}}}(x_{0}).\qquad\textrm{(by \eqref{eqn:smoothed})}
\end{align*}
Lemma \ref{lem} applied to the function \(f = (1 - \mathds{1}_{x_{0}}, \mathds{1}_{x_{0}})\) then implies that for \(x \in B_{L^{2}}(x_{0}, R)\), \(\widehat{\mathds{1}_{x_{0}}}(x) \geq 1/2\) and \(x_{0} \in \textrm{ProbLocEquiv}(X, f, R)\).
\end{proof}

\section{Plots of F-measure and Symmetric Best Dice Score as a Function of Training Epochs}
\label{app:performance}

\begin{figure}[H]
    \centering
    \begin{subfigure}{0.48\textwidth}
        \centering
        \includegraphics[width=\textwidth]{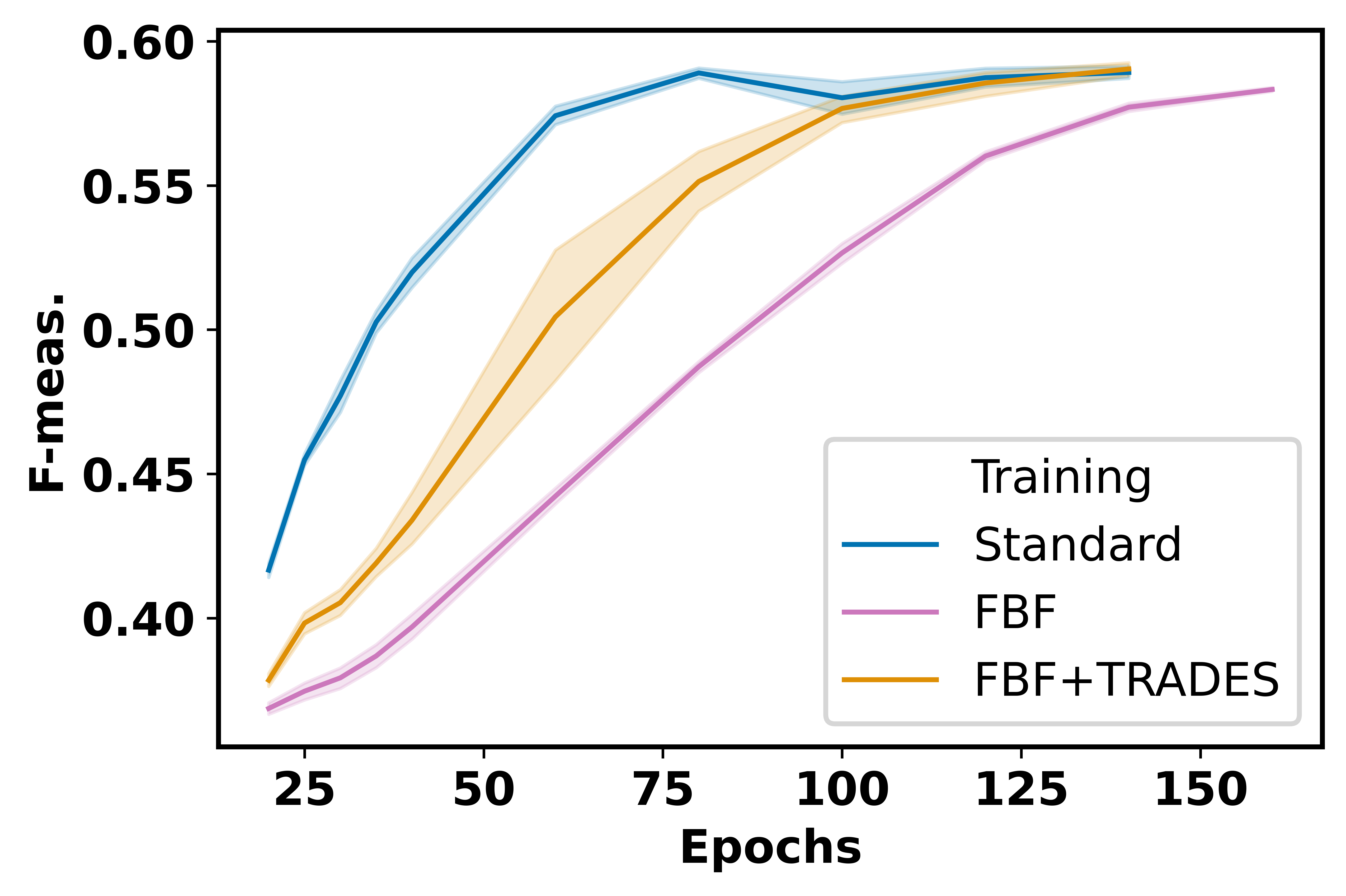}
    \end{subfigure}
    \hfill
    \begin{subfigure}{0.48\textwidth}
	  \centering
        \includegraphics[width=\textwidth]{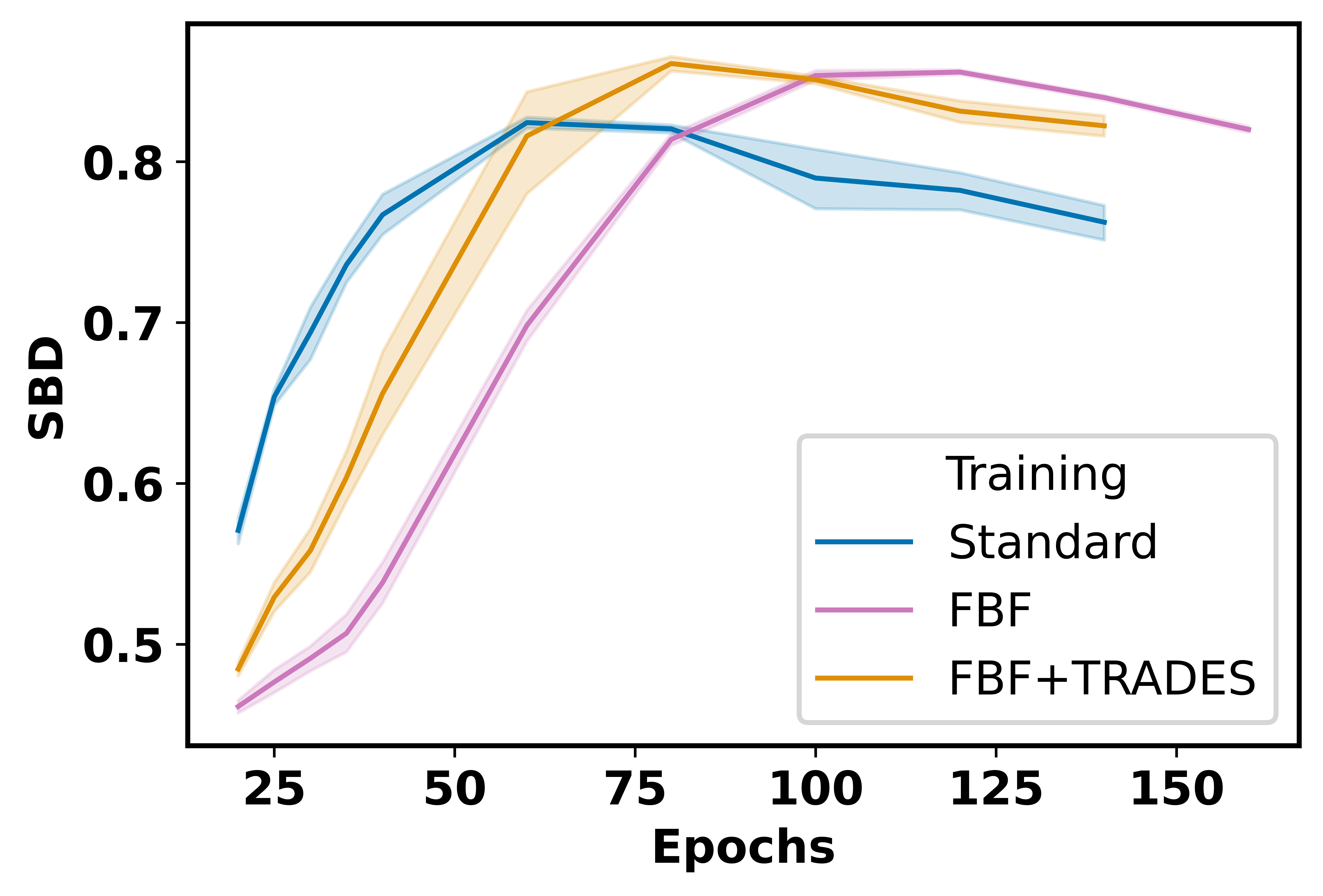}
    \end{subfigure}
    \caption{F-measure (left) and SBD (right) performance of the models on natural inputs from the validation set as a function of the number of training epochs}
    \label{fig:performance2}
\end{figure}

\begin{figure}[H]
    \centering
    \begin{subfigure}{0.48\textwidth}
        \centering
        \includegraphics[width=\textwidth]{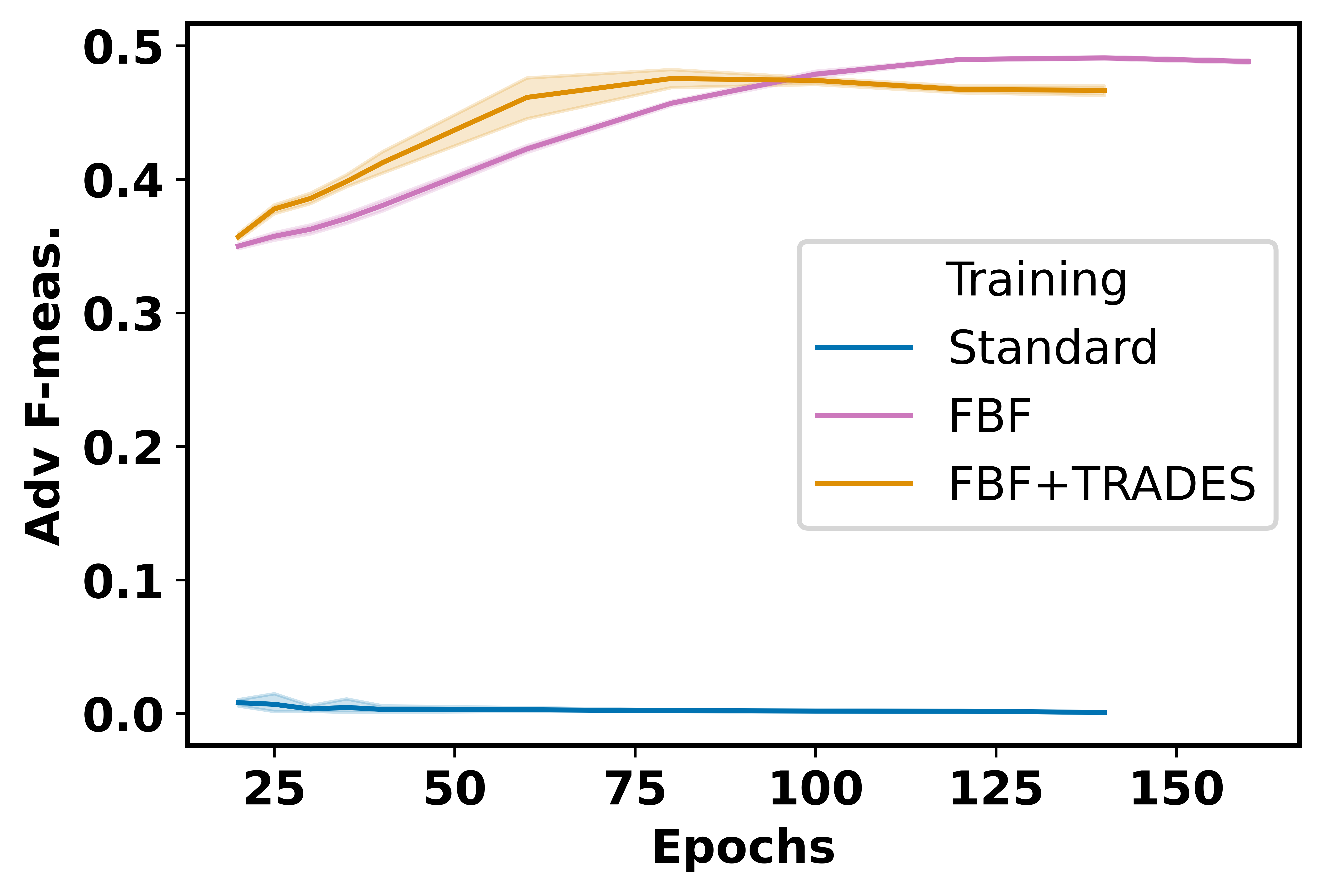}
    \end{subfigure}
    \hfill
    \begin{subfigure}{0.48\textwidth}
	  \centering
        \includegraphics[width=\textwidth]{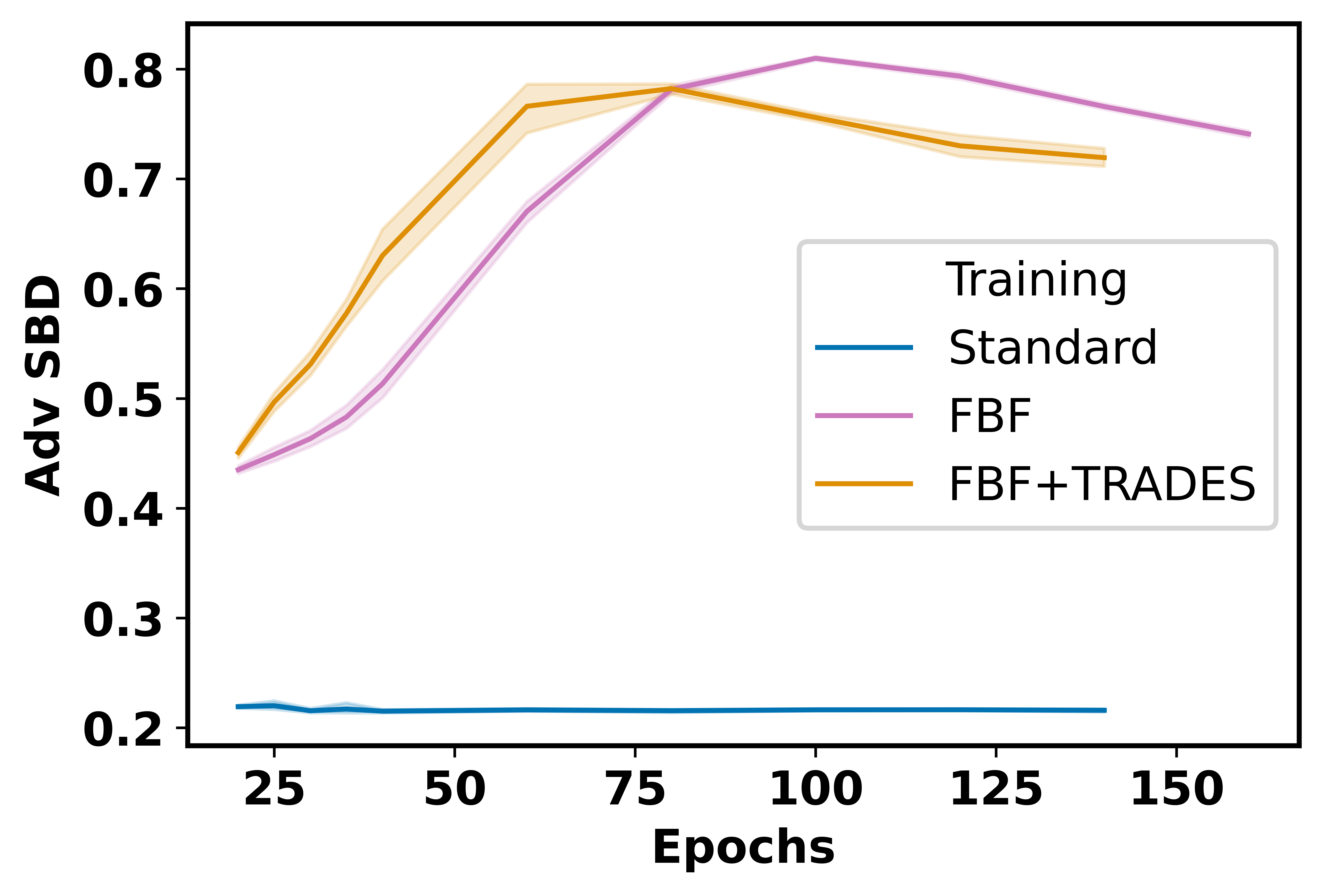}
    \end{subfigure}
    \caption{F-measure (left) and SBD (right) performance of the models on adversarial inputs from the validation set as a function of the number of training epochs}
    \label{fig:performance3}
\end{figure}

\section{Additional Adversarial Robustness Local Equivalence Plots}
\label{app:l2robustness}

\begin{figure}[H]
    \centering
    \includegraphics[width=0.65\textwidth]{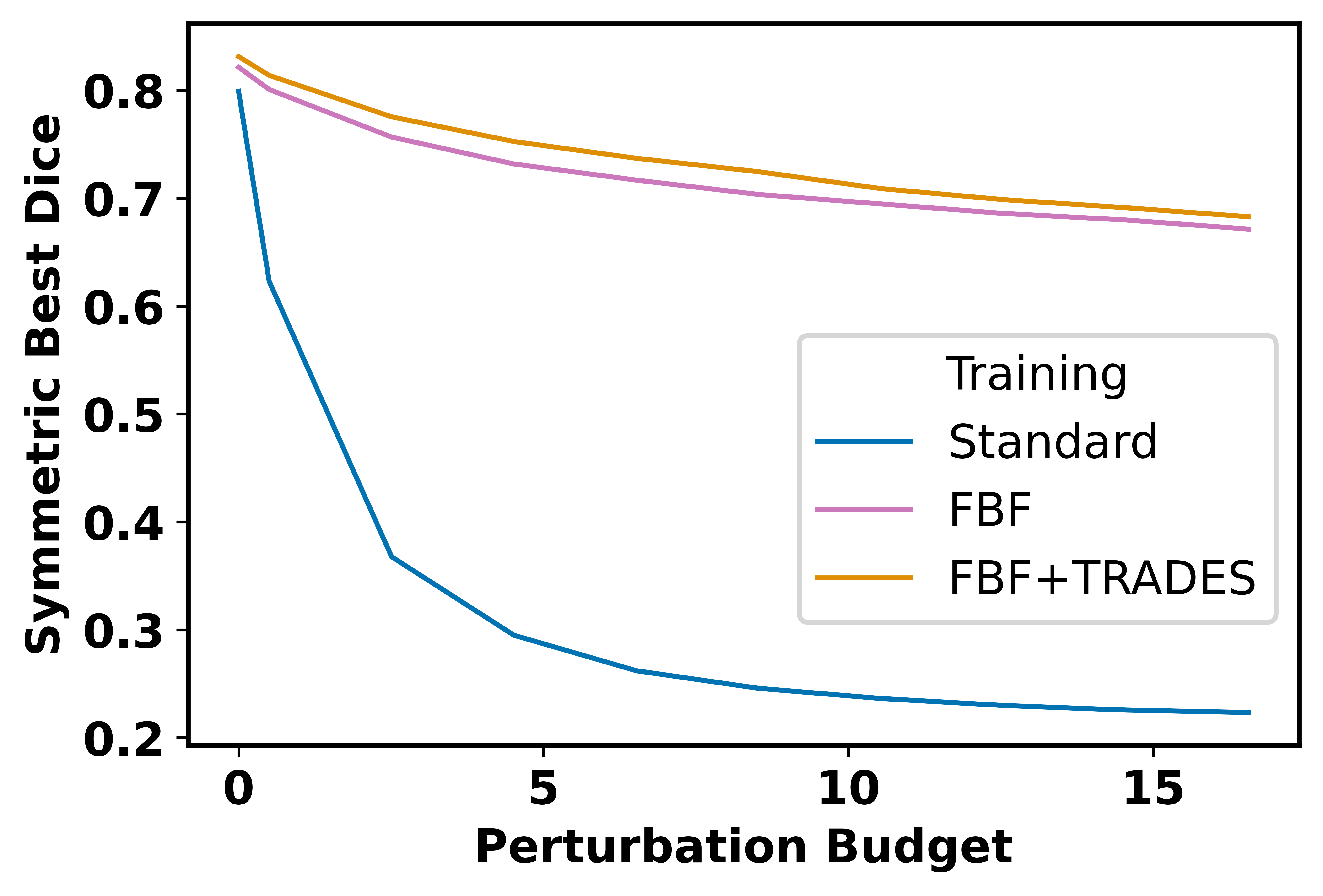}
    \caption{\(L^{2}\) security curve for each training method (10 steps of PGD)}
\end{figure}


\begin{figure}[H]
    \centering
    \includegraphics[width=0.65\textwidth]{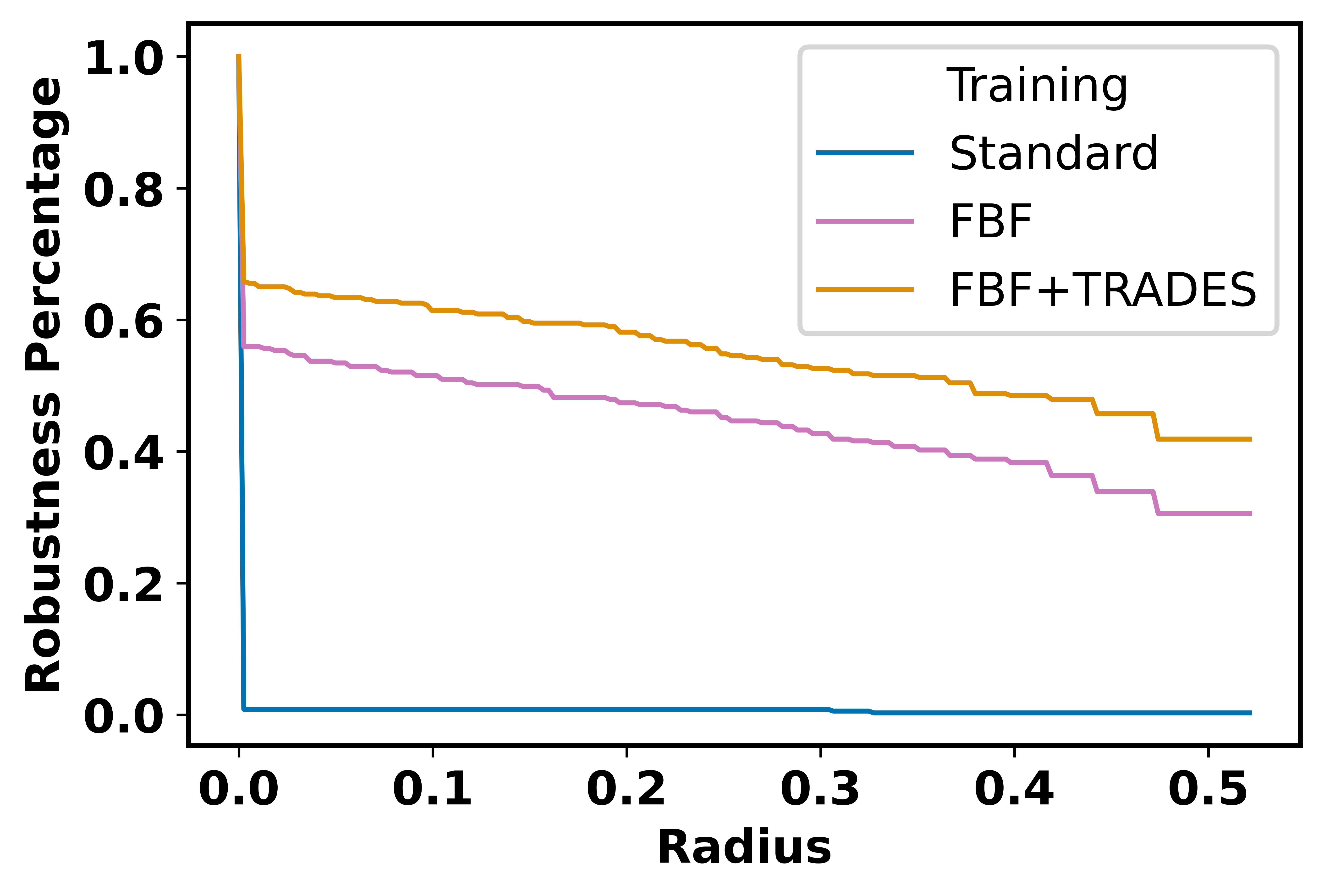}
    \caption{Robustness curve for each training method (\(\sigma=0.25, n=160\))}
\end{figure}


\section{Calculation of Total Training \& Evaluation Time}
\label{app:training}

\begin{table}[H]
\centering
\caption{Average training time for a given number of training epochs and training procedure}
\begin{tabular}{rcccccc}
\toprule
&\multicolumn{2}{c}{Standard}&\multicolumn{2}{c}{FBF}&\multicolumn{2}{c}{FBF+TRADES}\\
\cmidrule(lr){2-3}\cmidrule(lr){4-5}\cmidrule(lr){6-7}
&Avg. (min.)&\#&Avg. (min.)&\#&Avg. (min.)&\#\\
\midrule
20 Epochs&114.5&32&239.5&32&284.5&32\\
25 Epochs&141.5&16&300.0&32&353.0&32\\
30 Epochs&168.0&16&356.0&32&420.5&32\\
35 Epochs&194.0&16&411.5&16&486.0&16\\
40 Epochs&221.0&16&465.5&16&545.0&16\\
60 Epochs&327.0&16&696.5&16&831.5&16\\
80 Epochs&437.0&16&925.5&16&1087.5&16\\
100 Epochs&541.5&16&1184.5&16&1367.0&16\\
120 Epochs&758.0&16&1417.0&16&1636.0&16\\
140 Epochs&885.5&16&1620.5&16&1886.5&16\\
160 Epochs&&0&1878.5&16&&0\\
\bottomrule
\end{tabular}
\end{table}

\begin{table}[H]
\centering
\caption{Total time spent training and evaluating networks during final experiments}\
\begin{tabular}{rc}
\toprule
Phase & Total Time (min.)\\
\midrule
Training & \(413\,936\)\\[0.3cm]
\begin{minipage}{1.15in}
\begin{flushright}
Validation

(14 min./network)
\end{flushright}
\end{minipage}
& 8288\\[0.5cm]
\begin{minipage}{1.15in}
\begin{flushright}
Testing

(incl. Robustness/

Security curves)
\end{flushright}
\end{minipage}
& 6111\\
\midrule
Total & \(428\,335\)\\
\bottomrule
\end{tabular}
\end{table}

\section{Calculation of carbon emissions in terms of electric vehicle kilometers}
\label{app:carbon}

Based on figures on US carbon efficiency \cite{usprofile} and electric vehicle energy efficiency \cite{evemissions}, we find the following relationship between \si{kg.CO\textsubscript{2}} and \si{km} driven by an electric vehicle:

\[
\SI{1}{kg.CO\textsubscript{2}} \cdot
\frac{\SI{2.205}{lb}}{\SI{1}{kg}} \cdot
\frac{\SI{1}{MWh}}{\SI{853}{lb.CO\textsubscript{2}}} \cdot
\frac{\SI{1000}{kWh}}{\SI{1}{MWh}} \cdot
\frac{\SI{1}{mi}}{\SI{0.32}{kWh}} \cdot
\frac{\SI{1.609}{km}}{\SI{1}{mi}} = \SI{12.998}{km}
\]

\section{Effect of Number of PGD Iterations on Performance}
\label{app:steps}

\begin{figure}[H]
    \centering
    \includegraphics[width=0.65\textwidth]{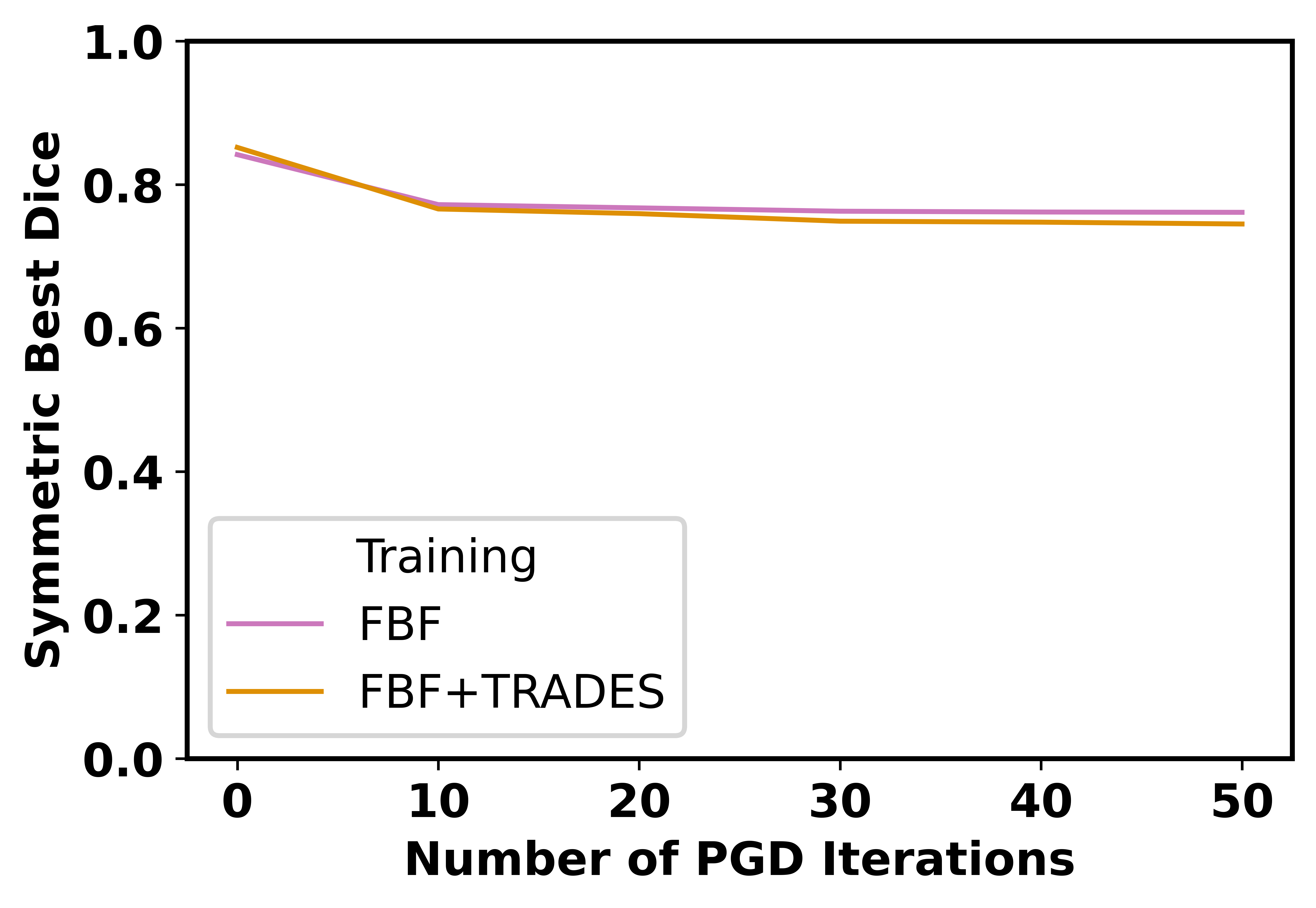}
    \caption{Model performance as a function of the number of PGD iterations (\(\varepsilon = 8/255\))}
\end{figure}

\section{Societal Impacts and Limitations}
\label{app:impacts}

The methods presented in this paper aim to improve the robustness of models, either by identifying non-robust models so their robustness can be improved or by training the models to be robust.  A positive impact of our work is in identifying non-robust models before their deployment can lead to harm \cite{balkanski20}.  The effect of the methods presented will hopefully be to improve the overall performance of models to which they are applied.  For systems with a positive impact on society, a more robust system should have a similarly positive impact.  However, the inverse is also true in that for a system with a detrimental impact on society, a more robust system will likely exacerbate this negative impact.


As a framework, probilistic local equivalence certification is highly customizable, requiring choices for similarity metric $\mathcal{M}$, robustness threshold $t$, and standard deviation $\sigma$, among other hyperparameters.  The flexibility of the framework can also become a limitation as it can be difficult to decide on appropriate choices for each of these hyperparameters in a principled manner.

\end{document}